\documentclass[]{elsarticle}

\usepackage{amsmath}
\usepackage{amsthm}
\usepackage{mathnotation}
\usepackage{ensuretext}
\usepackage{xcolor}
\usepackage{enumitem}

\usepackage{amssymb}

\journal{ArXiv.org}

\newcommand{\scQX}{{\textsc{QX}}}

\newcommand{\ma}{\mathcal{A}}

\newcommand{\mb}{\mathcal{B}}

\newcommand{\mc}{\mathcal{C}}

\newcommand{\ba}{\bar{\ma}}
\newcommand{\bb}{\bar{\mb}}
\newcommand{\bc}{\bar{\mc}}
\newcommand{\Inv}{\mathsf{Invar}}
\newcommand{\ddmb}{\ddot{\mb}}
\newcommand{\ddma}{\ddot{\ma}}
\newcommand{\ddmc}{\ddot{\mc}}
\newcommand{\dmb}{\dot{\mb}}
\newcommand{\dma}{\dot{\ma}}
\newcommand{\dmc}{\dot{\mc}}

\newcommand{\uu}[1]{\underline{\underline{#1}}}
\newcommand{\ssize}[1]{\mbox{\scriptsize #1}}

\definecolor{darkgreen}{RGB}{7,135,35}

\usepackage{algorithm}
\usepackage[noend]{algpseudocode}
\algrenewcommand\algorithmicrequire{\textbf{Input:}}
\algrenewcommand\algorithmicensure{\textbf{Output:}}

\newtheorem{definition}{Definition}[]{}
\newtheorem{proposition}{Proposition}[]{}
\newtheorem{lemma}{Lemma}[]{}
\newtheorem{theorem}{Theorem}[]{}
{}

\newcounter{examplecounter}
\newenvironment{example}{%\begin{myquote}%
	\refstepcounter{examplecounter}%
	
	\vspace{7pt}
	\noindent\textbf{Example \arabic{examplecounter}}%
	\quad 
}{
	
	\vspace{7pt}
}
\begin{document}
	
	\begin{frontmatter}
		
	\title{Understanding the \textsc{QuickXPlain} Algorithm: \\ Simple Explanation and Formal Proof\tnoteref{mytitlenote}}
	\tnotetext[mytitlenote]{This is a preprint of the work \protect\cite{rodler2022qx} that is formally published in the \emph{Artificial Intelligence Review (Artif. Intell. Rev.)} journal.}
	
%	\thanks{This is a preprint of the work \protect\cite{rodler2022qx} formally published in the \emph{Artificial Intelligence Review (Artif. Intell. Rev.)} journal.}

		%\title{\textsc{QuickXPlain} Revisited: \\ Understanding It and Proving It}
		%\title{A Formal Correctness Proof of the \\ QuickXPlain Algorithm}
		%\tnoteref{mytitlenote}}
		
		
		\author{Patrick Rodler}
		%\fnref{myfootnote}}
		\address{University of Klagenfurt \\ Universit\"atsstrasse 65-67, 9020 Klagenfurt}
		
		
		\begin{abstract}
			%In his seminal papers of 2001 and 2004, Ulrich Junker proposed the QuickXPlain Algorithm, which provides a divide-and-conquer computation strategy to find an irreducible subset with a particular (monotonic) property of a given set. 
			In his seminal 
			%papers of 2001 and 
			paper of
			2004, Ulrich Junker proposed the \textsc{QuickXPlain} algorithm, which provides a divide-and-conquer computation strategy to find within a given set an irreducible subset with a particular (monotone) property.
			Beside its original application in the domain of constraint satisfaction problems, the algorithm has since then found widespread adoption in areas as different as model-based diagnosis, recommender systems, verification, or the Semantic Web. This popularity is due to the frequent occurrence of the problem of finding irreducible subsets on the one hand, and to \textsc{QuickXPlain}'s general applicability and favorable computational complexity on the other hand. 
			%This popularity is owed to its general applicability on the one hand, and to its favorable computational complexity on the other hand. 
			
			However, although (we regularly experience) people are having a hard time understanding \textsc{QuickXPlain} and seeing why it works correctly, 
			%
			%even academics, let alone students, have a hard time understanding \textsc{QuickXPlain} due to  
			a proof of correctness of the algorithm has never been published.
			%, which is what we account for in this work.
			This is what we account for in this work, by 
			%presenting a novel way of 
			explaining \textsc{QuickXPlain} in a novel tried and tested way and by presenting an intelligible formal proof of it.   
			%In this work we account for that and show a formal proof of QuickXPlain. 
			Apart from showing the correctness of the algorithm and excluding the later detection of errors (\emph{proof and trust effect}), the added value of the availability of a formal proof is, e.g.,   
			%The added value of the availability of a formal proof 
			%is multi-dimensional, i.e., it has a ; apart from the proof effect, it has a didactic effect,   
			%of an algorithm 
			%is, first, 
			\emph{(i)} that the workings 
			%(and not only the soundness) 
			of the algorithm often become completely clear only after studying, verifying and comprehending the 
			%respective 
			proof (\emph{didactic effect}),  
			%(fosters algorithm understandability) and cognition
			%and, 
			%second, 
			\emph{(ii)} the shown proof methodology can be used as a guidance for proving other recursive algorithms (\emph{transfer effect}), and 
			\emph{(iii)} the possibility of providing ``gapless'' correctness proofs of systems that rely on (results computed by) \textsc{QuickXPlain}, such as numerous model-based debuggers 
			%that use QuickXPlain for the computation of intermediate results 
			(\emph{completeness effect}).
			%(provides point of reference for other works) 
			%TODO and, third, the proof showcases an approach as well as proof principles that are reuseable and generally applicable to other (recursive) algorithms.
			%(proof-didactics) 
			%for a (set-inclusion-)minimal subset of some universe which possesses a particular property.
		\end{abstract}
		
		\begin{keyword}
			\textsc{QuickXPlain} \sep Correctness Proof \sep Proof to Explain \sep Algorithm \sep Find Irreducible Subset with Monotone Property \sep MSMP Problem \sep Minimal Unsatisfiable Subset \sep Minimal Correction Subset \sep Model-Based Diagnosis \sep CSP
			%\MSC[2010] 00-01\sep  99-00
		\end{keyword}
		
	\end{frontmatter}
	
%	\linenumbers
	
	\section{Introduction}
	%(1) MSMP
	%-- important problem
	%-- can solve different problems in many various disciplines
	%-- examples
	%(2) QX
	%-- one popular and frequently adopted algorithm towards tackling MSMP is QX
	%-- however, no proof for QX
	%(3) why it is important that a proof is there
	%-- didactic effect, proving effect, completeness effect, transfer effect
	
	%The task of finding a minimal subset with a specific property of a given universe 
	The task of finding within a given universe an irreducible subset with a specific monotone property
	is referred to as the \emph{MSMP} (\emph{M}inimal \emph{S}et subject to a \emph{M}onotone \emph{P}redicate) \emph{problem} \cite{bradley2007checking,marques2013minimal}. Take the set of clauses $S:=\{\lnot C, A\lor \lnot B, C\lor \lnot B, \lnot A, B \}$ as an example. This set is obviously unsatisfiable. 
	%One task of interest---e.g., to understand the cause of the clauses' inconsistency---expressible as an MSMP problem is to find a minimal unsatisfiable subset (MUS) of these clauses.
	One task of interest expressible as an MSMP problem is to find a minimal unsatisfiable subset (MUS) of these clauses (which can help, e.g., to understand the cause of the clauses' inconsistency). 
	At this, $S$ is the \emph{universe}, and the predicate that tells whether 
	%a set of clauses given as input 
	a given set of clauses
	is satisfiable is \emph{monotone}, i.e., any superset (subset) of an unsatisfiable (satisfiable) clause set is unsatisfiable (satisfiable). In fact, there are two MUSes for $S$, i.e.,  $\{\lnot C, C\lor \lnot B, B\}$ and $\{A\lor \lnot B, \lnot A, B\}$. We call a task, such as MUS, that can be formulated as an MSMP problem a \emph{manifestation of the MSMP problem}. 
	
	MSMP  
	%Finding a minimal subset with a specific property of a given set is 
	%a very common problem in a 
	is 
	%of interest 
	relevant
	to a wide range of computer science disciplines, including model-based diagnosis \cite{jannach2016model,Rodler2015phd,
		%Shchekotykhin2014,
		Kalyanpur2006a,rodler2018statichs}, constraint satisfaction problems \cite{junker01,junker04,lecoutre2006recording}, verification \cite{bradley2007checking,bradley2008property,nadel2010boosting,andraus2008reveal}, configuration problems \cite{DBLP:journals/ai/FelfernigFJS04,white2010automated}, 
	%the Boolean satisfiability problem \cite{},
	knowledge representation and reasoning \cite{darwiche2001decomposable,mccarthy1980circumscription,eiter2009answer,marquis1995knowledge}, recommender systems \cite{felfernig2006integrated,felfernig2009utility}, knowledge integration \cite{Rodler2013,meilicke2011}, as well as description logics and the Semantic Web \cite{Kalyanpur2006a,rodler2019KBS_userstudy,Shchekotykhin2012,Horridge2011a,schlobach2007debugging,schekotihin2018ontodebug}. 
	%All these works address (sub-)problems which can be subsumed under the unified view of the MSMP problem.
	%
	%All these works address 
	In all these fields, (sub)problems are addressed which are manifestations of the MSMP problem.
	%, and hence can be solved by the same algorithms.
	%
	%More specifically, 
	%the mentioned works 
	%%in the context of these fields
	%tackle one of a multitude of questions related to the Boolean satisfiability problem that can be reduced to MSMP.
	%
	%Beside these areas, there is a multitude of questions related to the Boolean satisfiability problem that can be reduced to MSMP.
	% 
	Example problems---most of them related to the Boolean satisfiability problem---are the computation of \emph{minimal unsatisfiable subsets} \cite{marques2013minimal,dershowitz2006scalable,oh2004amuse,liffiton2008algorithms} (also termed \emph{conflicts} \cite{Reiter87,dekleer1987} or \emph{minimal unsatisfiable cores} \cite{dershowitz2006scalable}), \emph{minimal correction subsets} \cite{birnbaum2003consistent,marques2013computing} (also termed \emph{diagnoses} \cite{Reiter87,dekleer1987}), \emph{prime implicants} \cite{slagle1970new,quine1959cores} (also termed \emph{justifications} \cite{Horridge2011a}), 
	\emph{prime implicates} \cite{marquis1995knowledge,manquinho1997prime,deharbe2013computing}, 
	%\emph{minimal models}, 
	and \emph{most concise optimal queries to an oracle} \cite{Rodler2013,schekotihin2018ontodebug,DBLP:journals/corr/Rodler16a,DBLP:journals/corr/Rodler2017}. 
	%In model-based diagnosis, relevant MSMP problems are the finding of minimal conflicts 
	%
	% prime implicates
	%(PIs) (given an original implicate), minimal models (MMs), minimal unsatisﬁ-
	%able subsets (MUSes), minimal equivalent (or irredundant) subsets (MESes),
	%and minimal correction subsets (MCSes), among several others. 
	
	%All types of MSMP problems can be solved by the same algorithms. To this end, a range of algorithms \cite{bradley2007checking,bradley2008property,junker01,junker04,marques2013minimal,shchekotykhin2015mergexplain,felfernig2012efficient} have been independently suggested in literature, and most of them for some specific type of MSMP problem.
	%
	Numerous algorithms to solve manifestations of the MSMP problem have been suggested in literature, e.g.,  \cite{bradley2007checking,marques2013minimal,Rodler2015phd,junker01,junker04,bradley2008property,DBLP:journals/corr/Rodler2017,Shchekotykhin2014,shchekotykhin2015mergexplain,felfernig2012efficient,belov2012muser2}.
	For instance, the algorithm proposed by 
	Felfernig et al.\ \cite{felfernig2012efficient} addresses the problem of the computation of minimal correction subsets (diagnoses), 
	%Marquez-Silva et al.\ \cite{belov2012muser2} addresses the problem of the computation of minimal correction subsets (diagnoses),
	and the one suggested by Rodler et al.\ \cite{DBLP:journals/corr/Rodler2017} computes minimal oracle queries that preserve some optimality property. In general,
	% \cite{bradley2007checking,marques2013minimal}, 
	an algorithm $A$ for a specific manifestation of the MSMP problem can be used to solve arbitrary manifestations of the MSMP problem if 
	\emph{(i)}~the procedure used by $A$ to decide the monotone predicate is used as a black-box (i.e., given a subset of the universe as input, the procedure outputs 1 if the predicate is true for the subset and 0 otherwise; no more and no less), and 
	\emph{(ii)}~no assumptions or additional techniques are used in $A$ which are specific to one particular manifestation of the MSMP problem. 
	
	Not all algorithms meet these two criteria. For instance, there are algorithms that rely on additional outputs beyond the mere evaluation of the predicate (e.g., certificate-refinement-based algorithms \cite{marques2013minimal}), 
	%that rely on a certain subset of the universe called ``certificate'' that is returned beside the result of the predicate evaluation), 
	or glass-box approaches that use non-trivial modifications of the predicate decision procedure to solve the MSMP problem (e.g., theorem provers that record the axioms taking part in the deduction of a contradiction while performing a consistency check \cite{Kalyanpur2006a}). These methods violate (i). 
	%Regarding (ii), 
	%Moreover, there are, e.g., algorithms geared to the computation of minimal unsatisfiable subsets that leverage a technique called model rotation \cite{marques2011improving}. 
	%%which 
	%These algorithms are not applicable, e.g., to the problem of finding minimal correction subsets, since there is no concept equivalent to model rotation for minimal correction subsets \cite{marques2013minimal}. 
	Moreover, e.g., algorithms geared to the computation of minimal unsatisfiable subsets that leverage a technique called model rotation \cite{marques2011improving} are not applicable, e.g., to the problem of finding minimal correction subsets, since there is no concept equivalent to model rotation for minimal correction subsets \cite{marques2013minimal}.
	Thus, such algorithms violate (ii). 
	%a ``certificate'' like in 
	%certificate refinement approaches 
	%
	%Numerous algorithms to solve the MSMP problem have been suggested in literature \cite{bradley2007checking,marques2013minimal,Shchekotykhin2014,junker01,junker04,bradley2008property,DBLP:journals/corr/Rodler2017,shchekotykhin2015mergexplain,felfernig2012efficient}, most of them for some of the mentioned specific manifestations 
	%%(e.g., MUS) 
	%of the MSMP problem. For instance, the algorithm proposed by 
	%Felfernig et al.\ \cite{felfernig2012efficient} addresses the problem of the computation of minimal correction subsets (diagnoses), 
	%%Marquez-Silva et al.\ \cite{belov2012muser2} addresses the problem of the computation of minimal correction subsets (diagnoses),
	%and the one suggested by Rodler \cite{DBLP:journals/corr/Rodler2017} computes minimal oracle queries that preserve some optimality property. In general, an algorithm $A$ for a specific manifestation of the MSMP problem can be used to solve arbitrary manifestations of the MSMP problem if 
	%\emph{(i)}~the procedure used by $A$ to determine the monotone predicate is used as a black-box (i.e., given a subset of the universe as an input, the procedure outputs one if the predicate is true for the subset and zero otherwise; no more and no less), and 
	%\emph{(ii)}~no assumptions or additional techniques are used in $A$ which are specific to one particular manifestation of the MSMP problem.
	%
	%One such general MSMP algorithm satisfying (i) and (ii) among the referenced algorithms is \textsc{QuickXPlain}, which was proposed by Ulrich Junker in 2004 \cite{junker04}. 
	
	Among the general MSMP algorithms that satisfy (i) and (ii), \textsc{QuickXPlain} \cite{junker04} (QX for short), proposed by Ulrich Junker in 2004, is one of the most popular and most frequently adapted.\footnote{Judged by taking the citation tally on Google Scholar as a criterion; as of January 2020, the \textsc{QuickXPlain} paper boasts 420 citations.} 
	Likely reasons for the widespread use of QX are its mild theoretical complexity in terms of the number of (usually expensive\footnote{In many manifestations of the MSMP problem, predicate decision procedures are implemented by theorem provers, e.g., SAT-solvers \cite{marques2013minimal} or description logic reasoners \cite{Rodler2015phd}.}) predicate evaluations required \cite{marques2013minimal,junker04}, as well as its favorable practical performance for important problems (such as conflict \cite{shchekotykhin2008computing} or diagnosis \cite{Shchekotykhin2014} computation for model-based diagnosis).
	%\footnote{Note, that there are algorithms that outperform QX in particular applications (MSMP manifestations).} 
	In literature, QX is utilized in different ways; it is \emph{(a)}~\emph{(re)used as is} for suitable manifestations of MSMP \cite{DBLP:journals/ai/FelfernigFJS04}, \emph{(b)}~\emph{adapted} in order to solve other manifestations of MSMP \cite{Rodler2015phd}, as well as \emph{(c)}~\emph{modified or extended}, respectively, e.g., to achieve a better performance for a particular MSMP manifestation \cite{marques2013minimal}, to solve extensions of the MSMP problem \cite{DBLP:journals/corr/Rodler2017}, or to compute multiple minimal subsets of the universe in a single run \cite{shchekotykhin2015mergexplain}.
	
	Despite its popularity and common use, from the author's experience,\footnote{
		In our research and teaching on model-based diagnosis, we frequently discuss and analyze QX---one of the core algorithms used in our works and prototypes---with students as well as other faculty (including highly proficient university professors specialized in, e.g., algorithms and data structures). The feedback of people is usually that they cannot fully grasp the workings of QX before they take significant time to go through a particular example thoroughly and noting down all single steps of the algorithm. According to people's comments, the main obstacle appears to be the recursive nature of the algorithm.
		%	In our research and teaching on model-based diagnosis, we frequently discuss and analyze QX---which is one of the core algorithms used in our systems for the determination of irreducible faulty sets---together with both students and other teachers (including highly proficient university professors specialized in, e.g., algorithms and data structures). The feedback of people is usually that they cannot fully grasp the workings of QX before they take significant time to go through a particular example thoroughly and noting down all single steps of the algorithm. At this, the main obstacle appears to be the recursive nature of the algorithm.
	}
	QX appears to be quite poorly understood by reading and thinking through the algorithm, and, for most people, requires significant and time-consuming attention until they are able to properly explain the algorithm.
	%QX appears to be quite poorly understood by reading and thinking through the algorithm. For most people, it requires significant and time-consuming attention until they are able to properly explain the algorithm.
	In particular, people often complain they do not see why it correctly computes a minimal subset of the universe. This is not least because no proof of QX has yet been published. 
	
	%In this work, we account for that by presenting 
	%%and  illustrating 
	%a clear and intelligible proof of QX, after playing through a simple illustrating example using a notation that proved particularly comprehensible in our experience.\footnote{We (informally) experimented with different variants how to explain QX, and found out (through the feedback of discussion partners, e.g., students) that one representation was more accessible than others.} With that said, the main purposes and added values of this paper are:
	In this work, we account for this by presenting 
	a clear and intelligible proof of QX. 
	%Beside showing QX's correctness and making it verifiable in a step-by-step manner, the 
	The public availability of a proof comes with several benefits and serves i.a.\ the following purposes:
	% correctness of QX  of The availability of a clear and accessible proof of QX---beside showing the correctness of QX itself---comes i.a.\ with the following benefits
	%
	%With that said, the main purposes and added values of this paper are:
	\begin{description}[noitemsep,leftmargin=16pt]
		\item[Proof Effect] \emph{(a)} It shows  QX's correctness and makes it 
		verifiable for everyone in a straightforward step-by-step manner (without the need to accomplish the non-trivial task of coming up with an own proof). \emph{(b)} It creates compliance with common scientific practice. That is, every proposal of an algorithm
		%	---even more so for a highly influential one like QX---
		should be accompanied with a (full and public) formal proof of correctness. This demand is even more vital for a highly influential algorithm like QX.
		\item[Didactic Effect] \emph{(a)} 
		It promotes (proper and full) understanding \cite{hanna1996proof}
		%	It contributes to a full comprehension 
		of the workings of QX, which is otherwise for many people only possible in a laborious way (e.g., by noting down and exercising through examples and attempting to verify QX's soundness on concrete cases). \emph{(b)} It provides the basis for understanding (hundreds of) other works or algorithms that use, rely on, adapt, modify or extend QX.
		\item[Completeness Effect] It is necessary to establish and prove the full correctness of other algorithms that rely on (the correctness) of QX, such as a myriad of algorithms in the field of model-based diagnosis.
		%	\footnote{E.g., to prove the correctness of a diagnosis computation method, where diagnoses are computed via the computation of conflicts}
		\item[Trust and Sustainability Effect] It excludes the possibility of the (later) detection of flaws in the algorithm, and is thus the only basis for placing full confidence in the proper-functioning of QX.\footnote{A prominent example which shows that even seminal papers are not charmed against errors in absence of formal proofs, and thus underscores the importance of (public) proofs, is the highly influential paper of Raymond Reiter from 1987 \cite{Reiter87}. It proposes the hitting set algorithm for model-based diagnosis, but omits a formal proof of correctness. And, indeed, a critical error in the algorithm was later found (and corrected) by Russell Greiner \cite{greiner1989correction}.}
		\item[Transfer Effect] It showcases a stereotypic proof concept for recursive algorithms and can provide guidance to researchers when approaching the (often challenging task of formulating a) proof of other recursive algorithms. The reason is that recursive algorithms can often be proven using a similar methodology (as ours), e.g., by showing certain invariants and using a proof by induction.
	\end{description}

	The rest of this paper is organized as follows. We discuss related work in Sec.~\ref{sec:related}, before we briefly introduce the theoretical concepts required for the understanding and proof of QX in Sec.~\ref{sec:basics}. Then, in Sec.~\ref{sec:algo}, we state the QX algorithm in a 
	(slightly) more general formulation than originally published in \cite{junker04}, i.e., we present QX as a general method to tackle the MSMP problem.\footnote{The original algorithm was depicted specifically as a searcher for explanations or relaxations for over-constrained constraint satisfaction problems (CSPs). 
		%We, in contrast, present QX as a general method to tackle the MSMP problem. 
		Although the proper interpretation of the original formulation to address arbitrary MSMP problems different from CSPs may be relatively straightforward (for people familiar with CSPs), we believe that our more general depiction (cf.\ \cite{marques2013minimal}) can help readers non-familiar with the domain of CSPs to understand and correctly use QX without needing to properly re-interpret concepts from an unknown field.} 
	In addition, we explain the functioning of QX, and present an illustrative example using a notation that proved particularly comprehensible in our experience.\footnote{We (informally) experimented with different variants how to explain QX, and found out (through the feedback of discussion partners, e.g., students) that the shown representation was more accessible than others.} The proof is given in Sec.~\ref{sec:proof}, and concluding remarks are made in Sec.~\ref{sec:conclusion}.

	\section{Related Work}
	\label{sec:related}
	Bradley and Manna \cite{bradley2008property} state an algorithm claimed to be equivalent to QX and give a proof of this algorithm. However, first, there is no proof that the stated algorithm is indeed equivalent to QX (which is not clear 
	from the formulation given in \cite{bradley2008property}). Second, the proof given in \cite{bradley2008property} does not appear to be of great help to better understand QX, as the reader needs to become familiar with the notation and concepts used in \cite{bradley2008property} in the first place, and needs to map the pseudocode notation of \cite{bradley2008property} to the largely different one stuck to by Junker in the original QX-paper \cite{junker04}. Apart from that, the proof in \cite{bradley2008property}---despite (or perhaps exactly because of) its undeniable elegance---is not ``operation-centric'' in that it is not amenable to a mental ``tracking'' by means of the call-recursion-tree produced by QX. 
	%with regard to the execution 
	In contrast, our proof is \emph{illustrative} as it can be viewed as directly traversing the 
	%algorithm's 
	call-recursion-tree (cf.\ Fig.~\ref{fig:call_tree} later), while showing that certain invariant statements remain valid through all transitions in the tree, and using these invariants to prove that all (recursive) calls
	%---and thus also the initial call---of QX 
	work correctly. Moreover, we segment our proof into small, intuitive, and easily digestible chunks, thus putting a special \emph{focus on its clarity, elucidation, and didactic value}. Finally, our proof enables the verification of the correctness of the \emph{original formulation of QX}.
	%of the proof. 
	Hence, we believe that our proof is more valuable to people 
	%who are 
	having a hard time understanding QX than the one in \cite{bradley2008property}.  
	%given that people are having a hard time understanding QX, we argue that our proof can be of higher value to 
	%more elucidative wrt.\ people 
	Or, to put it into the words of Hanna \cite{hanna2000proof,hanna1990some}, we present a \emph{proof that explains}, rather one that solely proves.

	\section{Basics}
	\label{sec:basics}
	%%%%%%%%%%%%%%%%%%%%%%%%%%%%%%%%%%%%%%%%%%%%%%%%%%%%
	$\scQX$ can be employed to find, for an input set $U$, a minimal\footnote{Throughout this paper, minimality always refers to minimality wrt.\ set-inclusion.} 
	%(wrt.\ set-inclusion) 
	subset $X \subseteq U$ that has a certain monotone property $p$. An example would be an (unsatisfiable) knowledge base (set of logical sentences) $U$ for which we are interested in finding a minimal unsatisfiable subset (MUS) $X$.
	%for problems of completely different nature such as propositional unsatisfiability or over-constrainedness of constraint satisfaction problems. 
	%The only postulated prerequisite for $\scQX$ to work correctly is that $prop$ is a monotonic property. 
	%A property is monotonic if and only if the binary function that returns 1 if the property holds for the input set and 0 otherwise is a monotonic function.
	\begin{definition}[Monotone Property]\label{def:monotonic_property}
		Let $U$ be the universe (a set of elements) and $p:2^U \to \{0,1\}$ be a function where $p(X)=1$ iff property $p$ holds for $X\subseteq U$. 
		%	Let $X',X'' \subseteq U$. 
		Then, $p$ is a monotone property iff $p(\emptyset)=0$ and
		%	$X$ be a set and $f:2^X \rightarrow \setof{0,1}$ be a binary function defined for all subsets of $X$. Then, $f$ is monotonic iff 
		\begin{align*}
		\forall X', X'' \subseteq U:\;\, X' \subset X'' \implies p(X') \leq p(X'')
		\end{align*}
	\end{definition}
	So, $p$ is monotone iff, given that $p$ holds for some set $X'$, it follows that $p$ also holds for any superset $X''$ of $X'$. 
	%By simple logical transformation, an 
	An equivalent definition is: If $p$ does not hold for some set $X''$, $p$ does not hold for any subset $X'$ of $X''$ either.
	
	%In practical applications it is often a requirement that \emph{(a)}~some elements of the universe must not occur in the sought minimal subset, 
	%%with the property $p$, 
	%or \emph{(b)}~the minimal subset of the universe should be found in the context of some reference set. Both cases (a) and (b) can be subsumed as searching for a minimal subset of the \emph{analyzed set} $\ma$ given some \emph{background} $\mb$. In case (a), the background is defined as a subset of the universe (e.g., those sentences of a knowledge base that are assumed to be correct) and the analyzed set is constituted by all elements of the universe that are not in the background (those sentences that are possibly faulty); in case (b), the background is some additional set of relevance to the universe (e.g., a knowledge base of general medical knowledge, whereas the universe---here equal to the analyzed set---is a knowledge base describing a medical sub-discipline). 
	%For example, the problem of finding a MUS wrt.\ $\ma$ given background $\mb$ would be to search for a minimal set $X$ of elements in $\ma$ such that $X \cup \mb$ is unsatisfiable.
	%%An example would then be to search for a set including only elements of the universe $U$ which is a MUS for $U \cup \mb$. in the universe (whilst) which must not contain
	In practical applications it is often a requirement that \emph{(a)}~some elements of the universe must not occur in the sought minimal subset, 
	%with the property $p$, 
	or \emph{(b)}~the minimal subset of the universe should be found in the context of some reference set. Both cases (a) and (b) can be subsumed as searching for a minimal subset of the \emph{analyzed set} $\ma$ given some \emph{background} $\mb$. In case (a), $\mb$ is defined as a subset of the universe $U$ (e.g., in a fault localization task, those sentences of a knowledge base $U$ that are assumed to be correct) and $\ma$ is constituted by all other elements of the universe $U\setminus\mb$ (those sentences in $U$ that are possibly faulty); in case (b), $\mb$ is some additional set of relevance to the universe (e.g., a knowledge base of general medical knowledge), whereas $\ma$ is the universe itself (e.g., a knowledge base describing a medical sub-discipline). 
	For example, the problem of finding a MUS wrt.\ $\ma$ given background $\mb$ would be to search for a minimal set $X$ of elements in $\ma$ such that $X \cup \mb$ is unsatisfiable.
	%An example would then be to search for a set including only elements of the universe $U$ which is a MUS for $U \cup \mb$. in the universe (whilst) which must not contain
	\begin{definition}[$p$-Problem-Instance]\label{def:p-PI}
		Let $\ma$ (analyzed set) and $\mb$ (background) be (related) finite sets of elements where $\ma \cap \mb = \emptyset$, and let $p$ be a monotone predicate. Then we call the tuple $\tuple{\ma,\mb}$ a $p$-problem-instance ($p$-PI).
	\end{definition}
	%\begin{definition}[Minimal $p$-Set (given some Background)]\label{def:min_p-set}
	%	Let $\ma$ (analyzed set) and $\mb$ (background) be (related) sets of elements where $\ma \cap \mb = \emptyset$, and let $p$ be a monotonic predicate. Then, we call $X$ a \emph{$p$-set wrt.\ $\ma$ given $\mb$} iff $X\subseteq \ma$ and $p(X \cup \mb)=1$. We call a $p$-set $X$ wrt.\ $\ma$ given $\mb$ \emph{minimal} iff there is no $p$-set $X' \subset X$ wrt.\ $\ma$ given $\mb$. 
	%	
	%	For brevity, we write $X \in \pst{\ma,\mb}$ to state that $X$ is a $p$-set wrt.\ $\ma$ given $\mb$. 
	%	We call $X$ a $p$-set wrt.\ $\ma$, and write $X \in \pst{\ma}$, iff $X \in \pst{\ma,\emptyset}$.
	%%	
	%%	If no background $\mb$ is given, we call $X$ a $p$-set wrt.\ $\ma$ (written as $X \in \pst{\ma}$) iff $X \in \pst{\ma,\emptyset}$.
	%\end{definition}
	\begin{definition}[Minimal $p$-Set (given some Background)]\label{def:min_p-set}
		Let $\tuple{\ma,\mb}$ be a $p$-PI. Then, we call $X$ a \emph{$p$-set wrt.\ $\tuple{\ma,\mb}$} iff $X\subseteq \ma$ and $p(X \cup \mb)=1$. We call a $p$-set $X$ wrt.\ $\tuple{\ma,\mb}$ \emph{minimal} iff there is no $p$-set $X' \subset X$ wrt.\ $\tuple{\ma,\mb}$. 
		
		%	For brevity, we write $X \in \pst{\ma,\mb}$ to state that $X$ is a $p$-set wrt.\ $\tuple{\ma,\mb}$. 
		%	We call $X$ a $p$-set wrt.\ $\ma$, and write $X \in \pst{\ma}$, iff $X \in \pst{\ma,\emptyset}$.
		%	
		%	If no background $\mb$ is given, we call $X$ a $p$-set wrt.\ $\ma$ (written as $X \in \pst{\ma}$) iff $X \in \pst{\ma,\emptyset}$.
	\end{definition}
	Immediate consequences of Defs. \ref{def:monotonic_property} and \ref{def:min_p-set} are:
	\begin{proposition}[Existence of a $p$-Set]\label{prop:existence_of_p-set}
		\leavevmode
		\begin{enumerate}[noitemsep,label=(\arabic*)]
			\item \label{prop:existence_of_p-set:1} A (minimal) $p$-set exists for $\tuple{\ma,\mb}$ iff $p(\ma\cup\mb)=1$.
			\item \label{prop:existence_of_p-set:2} $\emptyset$ is a---and the only---(minimal) $p$-set wrt.\ $\tuple{\ma,\mb}$ iff $p(\mb)=1$.
		\end{enumerate}
	\end{proposition}
	
	\begin{algorithm}[t]
		\small
		\caption{$\scQX$: Computation of a Minimal $p$-Set} \label{algo:qx}
		\begin{algorithmic}[1]
			\Require a $p$-PI $\tuple{\ma,\mb}$ where $\ma$ is the analyzed set and $\mb$ is the background
			\Ensure a minimal $p$-set wrt.\ $\tuple{\ma,\mb}$, if existent; `no $p$-set', otherwise %\fixme{TODO}
			%\SetKwFunction{vr}{\normalfont \textsc{verifyRequirements}}
			%\SetKwFunction{generate}{\normalfont \textsc{findDiagnosis}}
			%\SetKwFunction{split}{\normalfont \textsc{split}}
			%\SetKwFunction{get}{\normalfont \textsc{getElements}}
			%\SetKwFunction{iscons}{\normalfont \textsc{isConsistent}}
			%\SetKwFunction{entails}{\normalfont \textsc{entails}}
			%\SetKwBlock{part}{function {\normalfont \textsc{findDiagnosis} ($\mb, \md, \Delta, \mo_\Delta, \mo, \Tne$)} returns {\normalfont a minimal diagnosis $\md$}}{end}
			%\SetKwBlock{ver}{function {\normalfont \textsc{verifyRequirements} ($\mb, \md, \mo, \Tne$)} returns {\normalfont \emph{true} or \emph{false}}}{end}
			
			\vspace{5pt}
			
			\Procedure{$\scQX$}{$\tuple{\ma,\mb}$}
			%\State $\mo' \gets \mo \setminus\mb$
			%\State $\mo' \gets \Call{sortAscendingByWeight}{\mo,w}$
			%\State $\mb' \gets \mb \cup \bigcup_{\tp\in\Tp} \tp$
			\If{$p(\ma\cup\mb)=0$} \label{algoline:validitytest1}
			\State \Return~`no $p$-set'  \label{algoline:return_no_p-set}
			\ElsIf{$\ma = \emptyset$}\label{algoline:O=0}
			\State \Return~$\emptyset$\label{algoline:emptyset}
			\Else \State \Return \Call{$\scQX'$}{$\mb, \tuple{\ma,\mb}$} \label{algoline:call_QX'}
			\EndIf
			\EndProcedure
			
			\vspace{5pt}
			
			\Procedure{$\scQX'$}{$\mc,\tuple{\ma,\mb}$}
			\If{$\mc \neq \emptyset \land p(\mb)=1$}\label{algoline:validitytest2} 
			% \fixme{\mb, (\emptyset, \Tp, \Tn, \RQ)} 
			\State \Return $\emptyset$ \label{algoline:return_emptyset}
			\EndIf
			\If{$|\ma| = 1$}  \label{algoline:test_singleton}              
			\State \Return $\ma$ \label{algoline:return_O}
			\EndIf
			\State $k \gets \Call{split}{|\ma|}$\label{algoline:split}
			\State $\ma_1 \gets \Call{get}{\ma, 1, k}$\label{algoline:get1} 
			\State $\ma_2 \gets \Call{get}{\ma, k + 1, |\ma|}$\label{algoline:get2}
			\State $X_2 \gets \Call{$\scQX'$}{\ma_1, \tuple{\ma_2,\mb\cup \ma_1}}$ \label{algoline:recursive_call1}
			\State $X_1 \gets \Call{$\scQX'$}{X_2, \tuple{\ma_1,\mb\cup X_2}}$ \label{algoline:recursive_call2}
			\State \Return $X_1 \cup X_2$ \label{algoline:return_upwards}
			\EndProcedure
			
		\end{algorithmic}
		\normalsize
	\end{algorithm}
	
	\section{Brief Review and Explanation of \scQX}
	\label{sec:algo}
	The QX algorithm is depicted by Alg.~\ref{algo:qx}. It gets as input a $p$-PI $\tuple{\ma,\mb}$ and assumes a sound and complete oracle that answers queries of the form $p(X)$ for arbitrary $X \subseteq \ma\cup\mb$. If existent, QX returns a minimal $p$-set wrt.\ $\tuple{\ma,\mb}$; otherwise, 'no $p$-set' is output.
	In a nutshell, QX works as follows:\vspace{5pt} 
	
	\noindent \emph{Trivial Cases:} 
	Before line~\ref{algoline:call_QX'} is reached, the algorithm checks if trivial cases apply, i.e., if either no $p$-set exists or the analyzed set $\ma$ is empty, and returns according outputs.
	In case the execution reaches line~\ref{algoline:call_QX'}, the recursive procedure QX' is called. 
	%In the very first execution of line~\ref{algoline:validitytest2} in QX', the presence of another trivial case is checked, i.e., if the background $\mb$ is non-empty and $p(\mb)=1$, then the empty set, the only minimal $p$-set in this case (cf.\ Prop.\ref{prop:existence_of_p-set}.\ref{prop:existence_of_p-set:2}), is directly returned and no recursive calls of QX' are made. Subsequently, in line~\ref{algoline:test_singleton},
	In the very first execution of 
	%line~\ref{algoline:validitytest2} in 
	QX', the presence of two other trivial cases is checked in lines~\ref{algoline:validitytest2} and \ref{algoline:test_singleton}. (Line~\ref{algoline:validitytest2}): If the background $\mb$ is non-empty and $p(\mb)=1$, then the empty set, the only minimal $p$-set in this case (cf.\ Prop.\ref{prop:existence_of_p-set}.\ref{prop:existence_of_p-set:2}), is directly returned and QX terminates. Otherwise, we know the empty set is not a $p$-set, i.e., every (minimal) $p$-set is non-empty. 
	% and no recursive calls of QX' are made. 
	(Line~\ref{algoline:test_singleton}): If the analyzed set $\ma$ is a singleton, then $\ma$ is directly returned and QX terminates. \vspace{0pt}
	
	\begin{figure}
		\centering
		\includegraphics[width=0.99\textwidth]{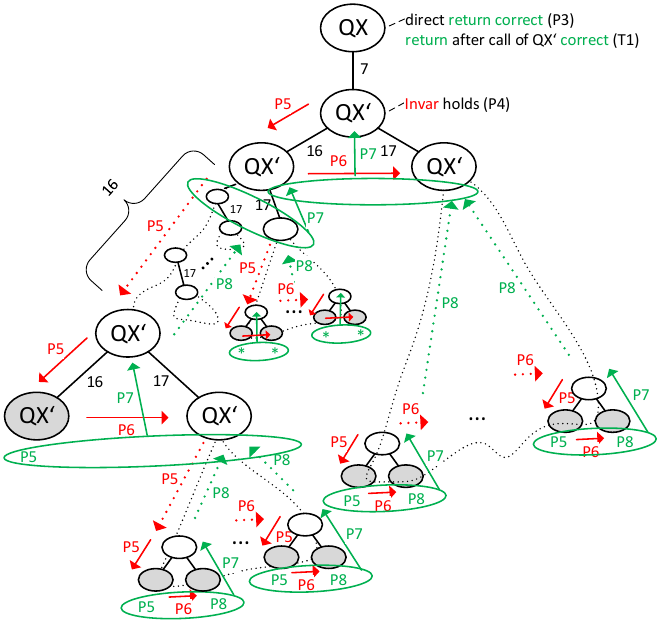}
		\caption{Call-recursion-tree produced by QX (cf.\ Sec.~\ref{sec:algo}). The \emph{grayscale part of the figure} provides a schematic illustration of the procedure calls executed in a single run of QX (where the recursion is entered, i.e., no trivial case applies). Each node (ellipse) represents one call of the procedure named within the ellipse. Edge labels (7,16,17) refer to the lines in Alg.~\ref{algo:qx} where the respective call is made. White ellipses (non-leaf nodes) are calls that issue further recursive calls (in lines 16 and 17), whereas gray ellipses (leaf nodes) are calls that directly return (i.e., in line~\ref{algoline:return_emptyset} or \ref{algoline:return_O}). The \emph{colored part of the figure} visualizes the meaning and consequences of the theorem (T1) and the various propositions (P$i$, for $i\in\{3,4,5,6,7,8\}$) that constitute the proof (cf.\ Sec.~\ref{sec:proof}). {\color{red}Red arrows} indicate proven propagations of the invariant property {\color{red} Invar} (see Def.~\ref{def:invariant_property}) between calls. {\color{darkgreen} Green arrows 
				and labels} indicate that respective calls return {\color{darkgreen} correct outputs}. Start to read the colored illustrations from the top, just like QX proceeds. That is, due to P3, direct returns yield correct outputs. If QX' is called, Invar 
			%the invariant Invar (see Def.~\ref{def:invariant_property}) 
			holds by P4. If Invar holds for some call, then it is always propagated downwards to the left subtree because of P5. At the first leaf node, a correct output is returned, also due to P5. If the output of a left subtree is correct, then Invar propagates to the right subtree (P6). If the output of both the left and the right subtree is correct, then the output of the root is correct (P7). If Invar holds at the root call of some (sub)tree, then this root call returns a correct output (P8). Note how these propositions guarantee that Invar, and thus correct outputs, can be derived for all nodes of the call-recursion-tree. Intuitively, red arrows propagate Invar downwards through the tree, which then ensures correct outcomes at the leafs, from where these correct outputs enable further propagation of Invar to the right, from where the inferred correct outputs are recursively propagated upwards until the root node is reached.}
		\label{fig:call_tree}
	\end{figure}
	
	\noindent \emph{Recursion:} Subsequently, the recursion is started. The principle is to partition the analyzed set $\ma = \{a_1,\dots,a_{|\ma|}\}$ into two \emph{non-empty} (e.g., equal-sized) subsets $\ma_1 = \{a_1,\dots,a_k\}$ and $\ma_2=\{a_{k+1},\dots,a_{|\ma|}\}$ (\textsc{split} and \textsc{get} functions; lines~\ref{algoline:split}--\ref{algoline:get2}), and to analyze these subsets recursively (\emph{divide-and-conquer}). In this vein, a binary call-recursion-tree is built (as sketched in Fig.~\ref{fig:call_tree}), including the \emph{root} QX'-call made in line~\ref{algoline:call_QX'} and \emph{two subtrees}, the left one rooted at the call of QX' in line~\ref{algoline:recursive_call1} which analyzes $\ma_2$, and the right one rooted at the call of QX' in line~\ref{algoline:recursive_call2} which analyzes $\ma_1$.
	Let the finally returned minimal $p$-set be denoted by $X$, and let us call all elements of $X$ \emph{relevant}, all others \emph{irrelevant}.
	Then, the left subtree (finally) returns the subset of those elements ($X_2$) from $\ma_2$ that belong to $X$, and the right subtree (finally) returns the subset of those elements ($X_1$) from $\ma_1$ that belong to $X$. 
	\begin{itemize}[leftmargin=*]
		\item \emph{Left subtree (recursive QX'-call in line~\ref{algoline:recursive_call1}):} The first question is: Are all elements of $\ma_2$ irrelevant? Or, equivalently: Does $\mb \cup \ma_1$ already contain a minimal $p$-set, i.e., $p(\mb\cup\ma_1)=1$? This is evaluated in line~\ref{algoline:validitytest2}; note: $\mc = \ma_1 \neq \emptyset$. If positive, 
		%	the empty set 
		$\emptyset$ is returned and the subtree is not further expanded. Otherwise, we know there is some relevant element in $\ma_2$. Hence, the analysis of $\ma_2$ is started. That is, in line~\ref{algoline:test_singleton}, the singleton test is performed for $\ma_2$. In the affirmative case, we have proven that the single element in $\ma_2$ is 
		%	in the minimal $p$-set. 
		relevant.
		The reason is that $p(\mb \cup \ma_1) = 0$, as verified in line~\ref{algoline:validitytest2} just before, and that adding the single element in $\ma_2$ makes the predicate true,\footnote{Such an element is commonly referred to as a \emph{necessary} or a \emph{transition} element \cite{belov2012muser2}.} i.e., $p(\mb\cup\ma_1\cup\ma_2) = p(\mb\cup\ma) = 1$, as verified in line~\ref{algoline:validitytest1} at the very beginning. If $\ma_2$ is a non-singleton, it is again partitioned and the subsets are analyzed recursively, which results in two new 
		%	respective 
		%	emanating edges in the call-recursion-tree and 
		subtrees in the call-recursion-tree.
		\item \emph{Right subtree (recursive QX'-call in line~\ref{algoline:recursive_call2}):} Here, we can distiguish between two possible cases, i.e., either the 
		set $X_2$ returned by the left subtree is \emph{(i)}~empty or \emph{(ii)}~non-empty. 
		\\Given (i), we know that $\ma_1$ must include a relevant element. Reason: $\mb\cup\ma_1$ contains a minimal $p$-set (as verified in the left subtree before returning the empty set) and every $p$-set is non-empty (as verified in line~\ref{algoline:validitytest2} in the course of checking the \emph{Trivial Cases}, see above). Hence, $\ma_1$ is further analyzed in lines~\ref{algoline:test_singleton} et seqq.\ (which might lead to a direct return if $\ma_1$ is a singleton and thus relevant, or to further recursive subtrees otherwise).
		\\For (ii), the question is: Given the subset $X_2$ of the $p$-set, are all elements of $\ma_1$ irrelevant? Or, equivalently: Does $\mb \cup X_2$ already contain a minimal $p$-set, i.e., $p(\mb \cup X_2)=1$? This is answered in line~\ref{algoline:validitytest2}; note: $\mc = X_2 \neq \emptyset$ due to case (ii). In the affirmative case, the empty set is returned, i.e., no elements of $\ma_1$ are relevant and the final $p$-set $X$ found by QX is equal to $X_2$. If the answer is negative, $\ma_1$ does include some relevant element and is thus further analyzed in lines~\ref{algoline:test_singleton} et seqq. (which might lead to a direct return if $\ma_1$ is a singleton and thus relevant, or to further recursive subtrees otherwise).  
		%	$\mb\cup\ma=\mb\cup\ma_1\cup\ma_2$ 
	\end{itemize} 
	Finally, the union of the outcomes of left ($X_2$) and right ($X_1$) subtrees is a minimal $p$-set wrt.\ $\tuple{\ma,\mb}$ and returned in line~\ref{algoline:return_upwards}. 
	%
	%, i.e., $p(X)=1$ iff 
	%That is, the query $p(X)=1$ ($p(X)=0$) is answered by true iff $p(X)=1$ holds (does not hold). 
	
	%The reason is, we tried various representations when explaining QX to people (mostly computer scientists) found out via people's feedback that the shown ``flat'' notation could best convey the intuition behind QX and enables people to correctly solve examples on their own.
	\begin{example}\label{ex:workings_qx}
		We illustrate the functioning of QX by means of a simple example.\vspace{5pt} 
		
		\noindent \emph{Input Problem and Parameter Setting:}
		Assume the analyzed set $\ma = \{1,2,3,4,$ $5,6,7,8\}$, the initial background $\mb = \emptyset$, and that there are two minimal $p$-sets wrt.\ $\tuple{\ma,\mb}$, $X=\{3,4,7\}$ and $Y=\{4,5,8\}$. Further, suppose that QX pursues a splitting strategy where a set is always partitioned 
		%	in  partitions the still relevant subset 
		into equal-sized subsets in each iteration, i.e., $\textsc{split}(n)$ returns $\lceil \frac{n}{2}\rceil$ (note: this leads to the best worst-case complexity of QX, cf.\ \cite{junker04}). \vspace{5pt} 
		
		\noindent\emph{Notation:} Below, we show the workings of QX on this example by means of a tried and tested ``flat'' notation.\footnote{Note, we intentionally abstain from a notation which is guided by the call-recursion-tree or which lists all variables and their values (which we found was often perceived difficult to understand, e.g., since same variable names are differently assigned in all the recursive calls). The reason is: While explaining QX to people (mostly computer scientists) using various representations, we found out via people's feedback that the presented ``flat'' notation could best convey the intuition behind QX; moreover, it enabled people to correctly solve new examples on their own.} 
		%	Note, we intentionally abstain from a notation which is guided by the call-recursion-tree or which lists all variables and their values (which we found was often perceived difficult to understand, e.g., since same variable names are differently assigned in all the recursive calls). Instead, we present a well-tried ``flat'' notation.
		%	%	Instead, we present a ``flat'' notation which we found could best convey the intuition behind QX (when we tried to explain QX to people, mostly computer scientists, using various representations) via people's feedback
		%	The reason is, we 
		%	%	tried to explain 
		%	explained
		%	QX to people (mostly computer scientists) using various representations and found out via people's feedback that this notation could best convey the intuition behind QX; moreover, it enabled people to correctly solve new examples on their own.
		%	
		In this notation, the single-underlined subset denotes the current input to the function $p$ in line~\ref{algoline:validitytest2}, the double-underlined elements are those that are already fixed elements of the returned minimal $p$-set, 
		%	$X$, 
		and the grayed out elements those that are definitely not in the returned minimal $p$-set.
		%	$X$. 
		Finally, $\textcircled{\scriptsize 1}$ signifies that the tested set (single-underlined along with double-underlined elements) is a $p$-set (function $p$ in line~\ref{algoline:validitytest2} returns 1); $\textcircled{\scriptsize 0}$ means it is no $p$-set (function $p$ in line~\ref{algoline:validitytest2} returns 0).\vspace{5pt}  
		
		\noindent \emph{How QX Proceeds:} After verifying that there is a non-empty $p$-set wrt.\ $\tuple{\ma,\mb}$ and that $|\ma|>1$ (i.e., after the checks in lines~\ref{algoline:validitytest1} and \ref{algoline:O=0} are negative, QX' is called in line~\ref{algoline:call_QX'}, and the checks in the first execution of lines~\ref{algoline:validitytest2} and \ref{algoline:test_singleton} are negative), 
		%	QX proceeds as follows:
		QX performs the following actions:
		\begin{align*}
		\mbox{\scriptsize (1)} \qquad [\underline{1, 2, 3, 4}, 5, 6, 7,  8] & \;\;\;\textcircled{\scriptsize 0} \quad \rightarrow \text{\quad some element of $p$-set among 5,6,7,8}\\
		\mbox{\scriptsize (2)} \qquad[\underline{1, 2, 3, 4, 5, 6}, 7,  8] & \;\;\;\textcircled{\scriptsize 0} \quad \rightarrow \text{\quad some element of $p$-set among 7,8}\\
		\mbox{\scriptsize (3)} \qquad[\underline{1, 2, 3, 4, 5, 6, 7},  8] & \;\;\;\textcircled{\scriptsize 1} \quad \rightarrow \text{\quad 7 found, 8 irrelevant}\\ 
		\mbox{\scriptsize (4)} \qquad[\underline{1, 2, 3, 4}, 5, 6, \uu{7},  {\color{lightgray}8}] & \;\;\;\textcircled{\scriptsize 1} \quad \rightarrow \text{\quad 5,6 irrelevant}\\
		\mbox{\scriptsize (5)} \qquad[1, 2, 3, 4, {\color{lightgray}5, 6}, \uu{7},  {\color{lightgray}8}] & \;\;\;\textcircled{\scriptsize 0} \quad \rightarrow \text{\quad some element of $p$-set among 1,2,3,4}\\
		\mbox{\scriptsize (6)} \qquad[\underline{1, 2}, 3, 4, {\color{lightgray}5, 6}, \uu{7},  {\color{lightgray}8}] & \;\;\;\textcircled{\scriptsize 0} \quad \rightarrow \text{\quad some element of $p$-set among 3,4}\\
		\mbox{\scriptsize (7)} \qquad[\underline{1, 2, 3}, 4, {\color{lightgray}5, 6}, \uu{7},  {\color{lightgray}8}] & \;\;\;\textcircled{\scriptsize 0} \quad \rightarrow \text{\quad 4 found}\\
		\mbox{\scriptsize (8)} \qquad[\underline{1, 2}, 3, \uu{4}, {\color{lightgray}5, 6}, \uu{7},  {\color{lightgray}8}] & \;\;\;\textcircled{\scriptsize 0} \quad \rightarrow \text{\quad 3 found}\\
		\mbox{\scriptsize (9)} \qquad[{\color{lightgray}1, 2}, \uu{3, 4}, {\color{lightgray}5, 6}, \uu{7},  {\color{lightgray}8}] & \;\;\;\textcircled{\scriptsize 1} \quad \rightarrow \text{\quad 1,2 irrelevant}
		\end{align*}
		\noindent \emph{Explanation:} After splitting $\ma$ into two subsets of equal size, in step \ssize{(1)}, QX tests if there is a $p$-set in the left half $\{1,2,3,4\}$. Since negative, the right half $\{5,6,7,8\}$ is again split into equal-sized subsets, and the left one $\{5,6\}$ is added to the left half $\{1,2,3,4\}$ of the original set. Because this larger set $\{1,2,3,4,5,6\}$ still does not contain any $p$-set, the right subset $\{7,8\}$ is again split and the left part (7) added to the tested set, yielding $\{1,2,3,4,5,6,7\}$. Due to the positive predicate-test for this set, 7 is confirmed as an element of the found minimal $p$-set, and 8 is irrelevant. From now on, 7, as a fixed element of the $p$-set, takes part in all further executed predicate tests. 
		
		In step \ssize{(4)}, the goal is to figure out whether the left half $\{5,6\}$ of $\{5,6,7,8\}$ also contains relevant elements. To this end, the left half $\{1,2,3,4\}$ of $\ma$, along with 7, is tested, and positive. Therefore, a $p$-set is included in $\{1,2,3,4,7\}$ and $\{5,6\}$ is irrelevant. At this point, the output of the left subtree of the root, the one that analyzed $\{5,6,7,8\}$, is determined and fixed, i.e., is given by 7. The next task is to find the relevant elements in the right subtree, i.e., among $\{1,2,3,4\}$. As a consequence, in step \ssize{(5)}, 7 alone is tested to check if all elements of $\{1,2,3,4\}$ are irrelevant. The result is negative, which is why the left half is split, and the left subset $\{1,2\}$ is tested along with 7, also negative. Thus, $\{3,4\}$ does include relevant elements. In step \ssize{(7)}, QX finds that the element 3 alone from the set $\{3,4\}$ does not suffice to produce a $p$-set, i.e., the test for $\{1,2,3,7\}$ is negative. This lets us conclude that 4 must be in the $p$-set. So, 4 is fixed. To check the relevance of 3, $\{1,2,4,7\}$ is tested, yielding a negative result, which proves that 3 is relevant. The final test in step \ssize{(9)} if $\{1,2\}$ includes relevant elements as well, is negative, and 1,2 marked irrelevant. The set $\{3,4,7\}$ is finally returned, which coincides with $X$, one of our minimal $p$-sets.  
		\qed	
	\end{example}

	\section{Proof of \scQX}
	\label{sec:proof}
	In this section, we give a formal proof of the termination and soundness of the $\scQX$ algorithm depicted by Alg.~\ref{algo:qx}. By ``soundness'' we refer to the property that $\scQX$ outputs a minimal $p$-set wrt.\ the $p$-PI it gets as an input, if a $p$-set exists, and 'no $p$-set' otherwise. While reading and thinking through the proof, the reader might consider it insightful to keep track of the meaning, implications, and interrelations of the various propositions in the proof by means of Fig.~\ref{fig:call_tree}.  
	
	\begin{proposition}[Termination]\label{prop:termination}
		Let $\tuple{\ma,\mb}$ be a $p$-PI. Then $\scQX(\tuple{\ma,\mb})$ terminates.	
	\end{proposition}
	\begin{proof}
		First, observe that $\scQX$ either reaches line~\ref{algoline:call_QX'} (where $\scQX'$ is called) or terminates before (in line~\ref{algoline:return_no_p-set} or line~\ref{algoline:emptyset}). Hence, $\scQX(\tuple{\ma,\mb})$ always terminates iff $\scQX'(\mb,\tuple{\ma,\mb})$ always terminates. We next show that $\scQX'(\mb,\tuple{\ma,\mb})$ terminates for an arbitrary $p$-PI $\tuple{\ma,\mb}$.
		
		$\scQX'(\mb,\tuple{\ma,\mb})$ either terminates directly (in that it returns in line~\ref{algoline:return_emptyset} or line~\ref{algoline:return_O}) or calls itself recursively in lines~\ref{algoline:recursive_call1} and \ref{algoline:recursive_call2}. 
		However, for each recursive call $\scQX'(\mc',\tuple{\ma',\mb'})$ within $\scQX'(\mb,\tuple{\ma,\mb})$ it holds that $\emptyset \subset \ma' \subset \ma$ as $\ma' \in \setof{\ma_1,\ma_2}$ (see lines~\ref{algoline:get1} and \ref{algoline:get2}) and $\emptyset \subset \ma_1, \ma_2 \subset \ma$ due to the definition of the \textsc{split} and \textsc{get} functions. 
		
		Now, assume an infinite sequence of nested recursive calls of $\scQX'$. Since $\ma$ is finite (Def.~\ref{def:p-PI}), this means that there must be a call $\scQX'(\bc,\tuple{\ba,\bb})$ in this sequence where $|\ba|=1$ and lines~\ref{algoline:recursive_call1} and \ref{algoline:recursive_call2} (next nested recursive call in the infinite sequence) are reached. This is a contradiction to the fact that the test in line~\ref{algoline:test_singleton} enforces a return in line~\ref{algoline:return_O} given that $|\ba|=1$. Consequently, every sequence of nested recursive calls during the execution of $\scQX'(\mb,\tuple{\ma,\mb})$ is finite (i.e., the depth of the call tree is finite).
		
		Finally, there can only be a finite number of such nested recursive call sequences because no more than two recursive calls are made in any execution of $\scQX'$ (i.e., the branching factor of the call tree is $2$). This completes the proof.
		%
		%Hence, every sequence of nested recursive calls of $\scQX'$ must 
		%
		%
		%Hence, each sequence of recursive call must finally reach the stopping criterion $|\mo|=1$ and return $\mo$ if it does not reach the stopping criterion in line~\ref{algoline:validitytest2} before.
	\end{proof}
	\noindent The following proposition witnesses that $\scQX$ is sound in case the sub-procedure $\scQX'$ is never called.
	\begin{proposition}[Correctness of $\scQX$ When Trivial Cases Apply]\label{prop:QX_correct_for_trivial_cases}
		\leavevmode
		\begin{enumerate}[noitemsep,label=(\arabic*)]
			\item \label{prop:QX_correct_for_trivial_cases:enum:no_p-set} $\scQX(\tuple{\ma,\mb})$ returns 'no $p$-set' in line~\ref{algoline:return_no_p-set} iff there is no $p$-set wrt.\ $\tuple{\ma,\mb}$.
			\item \label{prop:QX_correct_for_trivial_cases:enum:emptyset} If $\scQX(\tuple{\ma,\mb})$ returns $\emptyset$ in line~\ref{algoline:emptyset}, $\emptyset$ is a minimal $p$-set wrt.\ $\tuple{\ma,\mb}$.
			\item \label{prop:QX_correct_for_trivial_cases:enum:p(AuB)=1_if_line7_reached} If the execution of $\scQX(\tuple{\ma,\mb})$ reaches line~\ref{algoline:call_QX'}, $p(\ma\cup\mb)=1$ holds.
		\end{enumerate}
	\end{proposition}
	\begin{proof}
		We prove all statements \ref{prop:QX_correct_for_trivial_cases:enum:no_p-set}--\ref{prop:QX_correct_for_trivial_cases:enum:p(AuB)=1_if_line7_reached} in turn.\vspace{5pt}
		
		\noindent\emph{Proof of \ref{prop:QX_correct_for_trivial_cases:enum:no_p-set}:}  
		The fact follows directly from Prop.~\ref{prop:existence_of_p-set}.\ref{prop:existence_of_p-set:1} and the test performed in line~\ref{algoline:validitytest1}.\vspace{5pt}
		
		\noindent\emph{Proof of \ref{prop:QX_correct_for_trivial_cases:enum:emptyset}:}  
		Because line~\ref{algoline:emptyset} is reached, $p(\ma\cup\mb)=1$ (as otherwise a return would have taken place at line~\ref{algoline:return_no_p-set}) and $\ma = \emptyset$ (due to line~\ref{algoline:O=0}) must hold. Since $p(\ma\cup\mb)=1$ implies the existence of a $p$-set wrt.\ $\tuple{\ma,\mb}$ by Prop.~\ref{prop:existence_of_p-set}.\ref{prop:existence_of_p-set:1}, and since any $p$-set wrt.\ $\tuple{\ma,\mb}$ must be a subset of $\ma$ by Def.~\ref{def:min_p-set}, $\emptyset$ is the only (and therefore trivially a minimal) $p$-set wrt.\ $\tuple{\ma,\mb}$.\vspace{5pt}
		
		\noindent\emph{Proof of \ref{prop:QX_correct_for_trivial_cases:enum:p(AuB)=1_if_line7_reached}:}  
		This statement follows directly from the test in line~\ref{algoline:validitytest1} and the fact that line~\ref{algoline:call_QX'} is reached.
	\end{proof}
	We now characterize 
	%a property that constitutes 
	an invariant which applies to every call of $\scQX'$ throughout the execution of $\scQX$.  
	\begin{definition}[Invariant Property of $\scQX'$]\label{def:invariant_property}
		Let $\scQX'(\mc,\tuple{\ma,\mb})$ be a call of $\scQX'$. Then we say that $\Inv(\mc,\ma,\mb)$ holds for this call iff $$(\mc \neq \emptyset \lor p(\mb)=0) \;\,\land\;\, p(\ma \cup \mb) = 1$$
	\end{definition}
	\noindent The next proposition shows that this invariant holds for the first call of $\scQX'$ in Alg.~\ref{algo:qx}.
	\begin{proposition}[Invariant Holds For First Call of $\scQX'$]\label{prop:invar_holds_for_first_call_of_QX'}
		$\Inv(\bc,\ba,\bb)$ holds for $\scQX'(\bc,\tuple{\ba,\bb})$ given that $\scQX'(\bc,\tuple{\ba,\bb})$ was called in line~\ref{algoline:call_QX'}.
	\end{proposition}
	\begin{proof}
		Since $\scQX'(\bc,\tuple{\ba,\bb})$ was called in line~\ref{algoline:call_QX'}, we have $\bc = \mb$, $\ba = \ma$ and $\bb = \mb$. 
		Since $p(\ma\cup\mb)=1$ holds in line~\ref{algoline:call_QX'} on account of  Prop.~\ref{prop:QX_correct_for_trivial_cases}.\ref{prop:QX_correct_for_trivial_cases:enum:p(AuB)=1_if_line7_reached}, we have that $p(\ba\cup\bb)=1$. To show that $(\bc\neq\emptyset \lor p(\bb)=0)$, we distinguish the cases $\mb = \emptyset$ and $\mb \neq \emptyset$.
		Let first $\mb = \emptyset$. Due to Def.~\ref{def:monotonic_property}, we have that $p(\bb)=p(\mb)=p(\emptyset)=0$.
		Second, assume $\mb \neq \emptyset$. Since $\bc = \mb$, we directly obtain that $\bc \neq \emptyset$.	
	\end{proof}
	\noindent Given the invariant of Def.~\ref{def:invariant_property} holds for some call of $\scQX'$, we next demonstrate that the output returned by $\scQX'$ is sound (i.e., a minimal $p$-set) when it returns in line~\ref{algoline:return_emptyset} or \ref{algoline:return_O} (i.e., if this call of $\scQX'$ represents a leaf node in the call-recursion-tree). Moreover, we show that the invariant is ``propagated'' to the recursive call of $\scQX'$ in line~\ref{algoline:recursive_call1} (i.e., this invariant remains valid as long as the algorithm keeps going downwards in the call-recursion-tree). 
	\begin{proposition}[Invariant Causes Sound Outputs and Propagates Downwards]\label{prop:correctness_of_QX'_if_invar_and_propagation_of_invar}
		If $\Inv(\mc,\ma,\mb)$ holds for $\scQX'(\mc,\tuple{\ma,\mb})$, then:
		\begin{enumerate}[noitemsep,label=(\arabic*)]
			\item \label{prop:correctness_of_QX'_if_invar_and_propagation_of_invar:enum:return_line10_correct} $\scQX'(\mc,\tuple{\ma,\mb})$ returns $\emptyset$ in line~\ref{algoline:return_emptyset} iff $\emptyset$ is a (minimal) $p$-set wrt.\ $\tuple{\ma,\mb}$.
			\item \label{prop:correctness_of_QX'_if_invar_and_propagation_of_invar:enum:p(b)=0_after_line10} If the execution of $\scQX'(\mc,\tuple{\ma,\mb})$ reaches line~\ref{algoline:test_singleton}, then $p(\mb)=0$ holds.
			\item \label{prop:correctness_of_QX'_if_invar_and_propagation_of_invar:enum:return_line12_correct} If $\scQX'(\mc,\tuple{\ma,\mb})$ returns $\ma$ in line~\ref{algoline:return_O}, then $\ma$ is a minimal $p$-set wrt.\ $\tuple{\ma,\mb}$.
			\item 
			\label{prop:correctness_of_QX'_if_invar_and_propagation_of_invar:enum:propagation_of_invar_to_recursive_call} If the execution of $\scQX'(\mc,\tuple{\ma,\mb})$ reaches line~\ref{algoline:recursive_call1}, where $\scQX'(\bc,\tuple{\ba,\bb})$
			%		$\scQX'(\mc_{\ref{algoline:recursive_call1}},\tuple{\ma_{\ref{algoline:recursive_call1}},\mb_{\ref{algoline:recursive_call1}}})$ 
			is called, then $\Inv(\bc,\ba,\bb)$.
		\end{enumerate}
	\end{proposition}
	\begin{proof} We prove all statements \ref{prop:correctness_of_QX'_if_invar_and_propagation_of_invar:enum:return_line10_correct}--\ref{prop:correctness_of_QX'_if_invar_and_propagation_of_invar:enum:propagation_of_invar_to_recursive_call} in turn.\vspace{5pt}
		
		\noindent\emph{Proof of \ref{prop:correctness_of_QX'_if_invar_and_propagation_of_invar:enum:return_line10_correct}:} 
		``$\Rightarrow$'': We assume that $\scQX'(\mc,\tuple{\ma,\mb})$ returns in line~\ref{algoline:return_emptyset}. By the test performed in line~\ref{algoline:validitytest2}, 
		%a return in line~\ref{algoline:return_emptyset} can only take place 
		this can only be the case if $p(\mb)=1$. By Prop.~\ref{prop:existence_of_p-set}.\ref{prop:existence_of_p-set:2}, this implies that $\emptyset$ is a (minimal) $p$-set wrt.\ $\tuple{\ma,\mb}$.
		
		\noindent ``$\Leftarrow$'': We assume that $\emptyset$ is a (minimal) $p$-set wrt.\ $\tuple{\ma,\mb}$. To show that a return takes place in line~\ref{algoline:return_emptyset}, we have to prove that the condition tested in line~\ref{algoline:validitytest2} is true. First, we observe that $p(\mb)=1$ must hold due to Prop.~\ref{prop:existence_of_p-set}.\ref{prop:existence_of_p-set:2}. Since $\Inv(\mc,\ma,\mb)$ holds (see Def.~\ref{def:invariant_property}), we can infer from $p(\mb)=1$ that $\mc\neq \emptyset$. Hence, the condition in line~\ref{algoline:validitytest2} is satisfied. \vspace{5pt}
		
		\noindent\emph{Proof of \ref{prop:correctness_of_QX'_if_invar_and_propagation_of_invar:enum:p(b)=0_after_line10}:} 
		Prop.~\ref{prop:correctness_of_QX'_if_invar_and_propagation_of_invar}.\ref{prop:correctness_of_QX'_if_invar_and_propagation_of_invar:enum:return_line10_correct}
		%Statement \ref{prop:correctness_of_QX'_if_invar_and_propagation_of_invar:enum:return_line10_correct} of this Proposition 
		shows that line~\ref{algoline:test_singleton} is reached iff $\emptyset$ is not a $p$-set wrt.\ $\tuple{\ma,\mb}$ which is the case iff $p(\mb)=0$ due to Prop.~\ref{prop:existence_of_p-set}.\ref{prop:existence_of_p-set:2}.\vspace{5pt}
		
		\noindent\emph{Proof of \ref{prop:correctness_of_QX'_if_invar_and_propagation_of_invar:enum:return_line12_correct}:} 
		A return in line~\ref{algoline:return_O} can only occur if the test in line~\ref{algoline:test_singleton} is positive, i.e., if line~\ref{algoline:test_singleton} is reached and $|\ma|=1$. Moreover, since $\Inv(\mc,\ma,\mb)$ holds, it follows that $p(\ma\cup\mb)=1$. 
		
		First, $p(\ma\cup\mb)=1$ is equivalent to the existence of a $p$-set wrt.\ $\tuple{\ma,\mb}$. Second, by Def.~\ref{def:min_p-set}, a $p$-set wrt.\ $\tuple{\ma,\mb}$ is a subset of $\ma$. Third, $|\ma|=1$ means that $\emptyset$ and $\ma$ are all possible subsets of $\ma$. 
		%Therefore, the only possible $p$-sets are $\emptyset$ and $\ma$. 
		Fourth, since line~\ref{algoline:test_singleton} is reached, we have that $p(\mb)=0$ by statement~\ref{prop:correctness_of_QX'_if_invar_and_propagation_of_invar:enum:p(b)=0_after_line10} of this Proposition, which implies that $\emptyset$ is not a $p$-set wrt.\ $\tuple{\ma,\mb}$ according to Prop.~\ref{prop:existence_of_p-set}.\ref{prop:existence_of_p-set:2}. Consequently, $\ma$ must be a minimal $p$-set wrt.\ $\tuple{\ma,\mb}$.
		%
		%A return in line~\ref{algoline:return_O} can only occur if the test in line~\ref{algoline:test_singleton} is positive, i.e., if line~\ref{algoline:test_singleton} is reached and \emph{(i)} $|\ma|=1$. Since line~\ref{algoline:test_singleton} is reached, we have that $p(\mb)=0$ by statement~\ref{prop:correctness_of_QX'_if_invar_and_propagation_of_invar:enum:p(b)=0_after_line10} of this Proposition. Due to Prop.~\ref{prop:existence_of_p-set}.\ref{prop:existence_of_p-set:2}, we can deduce from this that \emph{(ii)} $\emptyset$ is not a $p$-set wrt.\ $\tuple{\ma,\mb}$. Moreover, since $\Inv(\mc,\ma,\mb)$ holds, it follows in particular that $p(\ma\cup\mb)=1$ is true. According to Prop.~\ref{prop:existence_of_p-set}.\ref{prop:existence_of_p-set:1} this however means that \emph{(iii)} a (minimal) $p$-set exists for $\tuple{\ma,\mb}$. Combining (i), (ii) and (iii) yields that $\ma\in\mpst{\ma,\mb}$ (since $\ma$ is the only non-empty subset of $\ma$)
		\vspace{5pt}
		
		\noindent\emph{Proof of \ref{prop:correctness_of_QX'_if_invar_and_propagation_of_invar:enum:propagation_of_invar_to_recursive_call}:} 
		Consider the call $\scQX'(\bc,\tuple{\ba,\bb})$ at line~\ref{algoline:recursive_call1}. Due to the definition of the \textsc{split} and \textsc{get} functions ($1\leq k \leq |\ma|-1$, $\ma_1$ includes the first $k$, $\ma_2$ the last $|\ma|-k$ elements of $\ma$) and the fact that $\bc = \ma_1$, the property $\bc \neq \emptyset$ must hold. Moreover, $\ba\cup\bb = \ma_2\cup\mb\cup\ma_1 = \ma\cup\mb$. Due to $\Inv(\mc,\ma,\mb)$, however, we know that $p(\ma\cup\mb)=1$. Therefore, $p(\ba\cup\bb)=1$ must be true. According to Def.~\ref{def:invariant_property}, it follows that $\Inv(\bc,\ba,\bb)$ holds.
	\end{proof}
	\noindent Note, immediately before line~\ref{algoline:recursive_call2} is first reached during the execution of $\scQX$, it must be the case that, for the first time, a recursive call 
	%$\scQX'(\bc,\tuple{\ba,\bb})$ 
	$\scQX'(\mc,\tuple{\ma,\mb})$ 
	made in line~\ref{algoline:recursive_call1} returns (i.e., we reach a leaf node in the call-recursion-tree for the first time and the first backtracking takes place). By Prop.~\ref{prop:correctness_of_QX'_if_invar_and_propagation_of_invar}.\ref{prop:correctness_of_QX'_if_invar_and_propagation_of_invar:enum:return_line10_correct}+\ref{prop:correctness_of_QX'_if_invar_and_propagation_of_invar:enum:return_line12_correct}, the output of this call $\scQX'(\mc,\tuple{\ma,\mb})$, namely $X_2$ in line~\ref{algoline:recursive_call1}, is a minimal $p$-set wrt.\ $\tuple{\ma,\mb}$. We now prove that the invariant property given in Def.~\ref{def:invariant_property} in this case ``propagates'' to the first-ever call of $\scQX'$ in line~\ref{algoline:recursive_call2}.
	\begin{proposition}[If Output of Left Sub-Tree is Sound, Invariant Propagates to Right Sub-Tree]\label{prop:invar_propagates_to_first_call_in_line17}
		Let $\Inv(\mc,\ma,\mb)$ be true for some call $\scQX'(\mc,\tuple{\ma,\mb})$ and let the recursive call $\scQX'(\dmc,\langle\dma,\dmb\rangle)$ in line~\ref{algoline:recursive_call1} during the execution of $\scQX'(\mc,\tuple{\ma,\mb})$ return a minimal $p$-set wrt.\ $\langle\dma,\dmb\rangle$. Then $\Inv(\ddmc,\ddma,\ddmb)$ holds for the recursive call $\scQX'(\ddmc,\langle\ddma,\ddmb\rangle)$ in line~\ref{algoline:recursive_call2} during the execution of $\scQX'(\mc,\tuple{\ma,\mb})$.
	\end{proposition}
	\begin{proof}
		As per Def.~\ref{def:invariant_property}, we have to show that $(\ddmc \neq \emptyset \lor p(\ddmb)=0) \land p(\ddma \cup \ddmb) = 1$.\vspace{5pt} 
		
		\noindent We first prove $p(\ddma \cup \ddmb) = 1$. Since $X_2$, the set returned by $\scQX'(\dmc,\langle\dma,\dmb\rangle) = \scQX'(\ma_1,\langle\ma_2,\mb\cup\ma_1\rangle)$ in line~\ref{algoline:recursive_call1}, is a minimal $p$-set wrt.\ $\langle\dma,\dmb\rangle = \langle\ma_2,\mb\cup\ma_1\rangle$, we infer by Def.~\ref{def:min_p-set} that $p(X_2 \cup \mb \cup \ma_1) = 1$. However, it holds that $\scQX'(\ddmc,\langle\ddma,\ddmb\rangle)=\scQX'(X_2,\tuple{\ma_1,\mb\cup X_2})$. Therefore, $p(\ddma\cup\ddmb) = p([\ma_1]\cup[\mb\cup X_2]) = 1$.\vspace{5pt}
		
		\noindent It remains to be shown that $(\ddmc \neq \emptyset \lor p(\ddmb)=0)$ holds, which is equivalent to $(X_2 \neq \emptyset \lor p(\mb \cup X_2)=0)$. If $X_2 \neq \emptyset$, we are done. So, let us assume that $X_2 = \emptyset$. In this case, however, we have $p(\mb \cup X_2)=p(\mb)$. As $\Inv(\mc,\tuple{\ma,\mb})$ holds and line~\ref{algoline:recursive_call2} is reached during the execution of $\scQX'(\mc,\tuple{\ma,\mb})$, we know by Prop.~\ref{prop:correctness_of_QX'_if_invar_and_propagation_of_invar}.\ref{prop:correctness_of_QX'_if_invar_and_propagation_of_invar:enum:p(b)=0_after_line10} that $p(\mb)=0$. Hence, 
		%	$p(\ddmb) = 
		$p(\mb \cup X_2) = p(\mb) = 0$.\vspace{5pt}
		
		\noindent Overall, we have demonstrated that $\Inv(\ddmc,\langle\ddma,\ddmb\rangle)$ holds.
	\end{proof}
	\noindent At this point, we know that the invariant property of Def.~\ref{def:invariant_property} remains valid up to and including the first recursive call of $\scQX'$ in line~\ref{algoline:recursive_call2} (i.e., until immediately after the first leaf in the call-recursion-tree is encountered, a single-step backtrack is made, and the first branching to the right is executed). From then on, as long as only ``downward'' calls of $\scQX'$ in line~\ref{algoline:recursive_call1}, possibly interleaved with single calls of $\scQX'$ in line~\ref{algoline:recursive_call2}, are performed, the validity of the invariant is preserved. 
	
	Due to the fact that $\scQX$ terminates (Prop.~\ref{prop:termination}), the call-recursion-tree must be finite. Hence, the situation must occur, where $\scQX'$ called in line~\ref{algoline:recursive_call1} directly returns (i.e., in line~\ref{algoline:return_emptyset} or \ref{algoline:return_O}) and the immediately subsequent call of $\scQX'$ in line~\ref{algoline:recursive_call2} directly returns (i.e., in line~\ref{algoline:return_emptyset} or \ref{algoline:return_O}) as well (i.e., we face the situation where both the left and the right branch at one node in the call-recursion-tree consist only of a single leaf node). As the invariant holds in this right branch, the said call of $\scQX'$ in line~\ref{algoline:recursive_call2} must indeed return a minimal $p$-set wrt.\ its $p$-PI given as an argument, due to Prop.~\ref{prop:correctness_of_QX'_if_invar_and_propagation_of_invar}.\ref{prop:correctness_of_QX'_if_invar_and_propagation_of_invar:enum:return_line10_correct}+\ref{prop:correctness_of_QX'_if_invar_and_propagation_of_invar:enum:return_line12_correct}.
	
	The next proposition evidences---as a special case---that the combination (set-union) of the two outputs $X_2$ (left leaf node) and $X_1$ (right leaf node) returned in line~\ref{algoline:return_upwards} in fact constitutes a minimal $p$-set for the $p$-PI given as an input argument to the call of $\scQX'$ which executes line~\ref{algoline:return_upwards}. More generally, the proposition testifies that, given the calls in line~\ref{algoline:recursive_call1} and line~\ref{algoline:recursive_call2} each return a minimal $p$-set wrt.\ their given $p$-PIs---whether or not these calls directly return---the combination of these $p$-sets is again a minimal $p$-set for the respective $p$-PI at the call that executed lines~\ref{algoline:recursive_call1} and \ref{algoline:recursive_call2}.\footnote{Note, this proposition is stated in \cite{junker04}, but not proven.} 
	%%%%%% V %%%%%%%%
	\begin{proposition}[If Output of Both Left and Right Sub-Tree is Sound, then a Sound Result is Returned (Propagated Upwards)]\label{prop:if_both_subtrees_sound_then_tree_is_sound}
		Let the recursive call $\scQX'(\dmc,\langle\dma,\dmb\rangle)$ in line~\ref{algoline:recursive_call1} during the execution of $\scQX'(\bc,\tuple{\ba,\bb})$ return 
		%$X_2$, 
		a minimal $p$-set wrt.\ $\langle\dma,\dmb\rangle$, and let the recursive call $\scQX'(\ddmc,\langle\ddma,\ddmb\rangle)$ in line~\ref{algoline:recursive_call2} during the execution of $\scQX'(\bc,\tuple{\ba,\bb})$ return 
		%$X_1$, 
		a minimal $p$-set wrt.\ $\langle\ddma,\ddmb\rangle$. Then $\scQX'(\bc,\tuple{\ba,\bb})$ returns a minimal $p$-set wrt.\  $\tuple{\ba,\bb}$.	
	\end{proposition}
	\begin{proof}
		The statement is a direct consequence of Lemma~\ref{lem:qx_recursion_principle} below.
	\end{proof}
	\begin{lemma}\label{lem:qx_recursion_principle}
		%	Let $\ma_1, \ma_2$ be a partition of $\ma$. If \emph{(a)}~$X_2 \in \mpst{\ma_2, \mb \cup \ma_1}$ and \emph{(b)}~$X_1\in\mpst{\ma_1, \mb \cup X_2}$, then $X_1 \cup X_2\in
		%%	\mpst{\ma_1 \cup \ma_2, \mb}=
		%	\mpst{\ma, \mb}$.
		Let $\ma_1, \ma_2$ be a partition of $\ma$. If \emph{(a)}~$X_2$ is a minimal $p$-set wrt.\ $\tuple{\ma_2, \mb \cup \ma_1}$ and \emph{(b)}~$X_1$ is a minimal $p$-set wrt.\ $\tuple{\ma_1, \mb \cup X_2}$, then $X_1 \cup X_2$ is a minimal $p$-set wrt.\ $\tuple{\ma, \mb}$.
	\end{lemma}
	\begin{proof} We first show that $X_1 \cup X_2$ is a $p$-set, and then we show its minimality.\vspace{5pt}
		
		\noindent\emph{$p$-set property:}
		First, by Def.~\ref{def:min_p-set}, $X_1 \subseteq \ma_1$ due to (a), and $X_2 \subseteq \ma_2$ due to (b), which is why $X_1 \cup X_2 \subseteq \ma_1 \cup \ma_2 = \ma$.
		From the fact that $X_1$ is a minimal $p$-set wrt.\ $\tuple{\ma_1, \mb \cup X_2}$, along with Def.~\ref{def:min_p-set}, we get $p(X_1 \cup [\mb \cup X_2]) = 1 = p([X_1 \cup X_2] \cup \mb)$. Hence, 
		%	$X_1 \cup X_2 \in \pst{\ma,\mb}$ 
		$X_1 \cup X_2$ is a $p$-set wrt.\ $\tuple{\ma,\mb}$ 
		due to Def.~\ref{def:min_p-set}.\vspace{5pt} 
		
		\noindent\emph{Minimality:}
		To show that $X_1 \cup X_2$ is a \emph{minimal} $p$-set wrt.\ $\tuple{\ma,\mb}$, assume that $X \subset X_1 \cup X_2$ is a $p$-set wrt.\ $\tuple{\ma,\mb}$. 
		%	Due to $\ma_1 \cap \ma_2 = \emptyset$ and $X_1 \subseteq \ma_1$ and $X_2 \subseteq \ma_2$, it must hold that $X_1\cap X_2 = \emptyset$. 
		The set $X$ can be represented as $X=X'_1 \cup X'_2$ where (1)~$X'_1 := X \cap X_1 \subseteq X_1$ and (2)~$X'_2 := X \cap X_2 \subseteq X_2$. In addition, the $\subseteq$-relation in (1) or (2) must be a $\subset$-relation, i.e., $X$ does not include all elements of $X_1$ or not all elements of $X_2$.
		
		Let us first assume that $\subset$ holds in (1). Then, $X= X'_1 \cup X'_2$ where $X'_1 \subset X_1$ and $X'_2 \subseteq X_2$. Since $X$ is a $p$-set wrt.\ $\tuple{\ma,\mb}$, we have $p(X\cup \mb) = p([X'_1 \cup X'_2] \cup \mb) = p(X'_1 \cup [\mb \cup X'_2])=1$. By monotonicity of $p$, it follows that $p(X'_1 \cup [\mb \cup X_2])=1$. Because of $X'_1 \subset X_1 \subseteq \ma_1$, we have that $X'_1$ is a $p$-set wrt.\ $\tuple{\ma_1,\mb\cup X_2}$, which is a contradiction to the premise (b). % $X_1 \in \mpst{\ma_1,\mb \cup X_2}$. 
		
		Second, assume that $\subset$ holds in (2). Then, $X= X'_1 \cup X'_2$ where $X'_1 \subseteq X_1$ and $X'_2 \subset X_2$. Since $X$ is a $p$-set wrt.\ $\tuple{\ma,\mb}$, we have $p(X\cup \mb) = p([X'_1 \cup X'_2] \cup \mb) = p(X'_2 \cup [\mb \cup X'_1])=1$. By monotonicity of $p$, and since $X'_1 \subseteq X_1 \subseteq \ma_1$, it follows that $p(X'_2 \cup [\mb \cup \ma_1])=1$. As $X'_2 \subset X_2 \subseteq \ma_2$, we obtain that $X'_2$ is a $p$-set wrt.\ $\tuple{\ma_2,\mb \cup \ma_1}$, which is a contradiction to premise (a). 
	\end{proof}
	
	\begin{proposition}[If Invariant Holds for Tree, Then a Minimal $p$-Set is Returned By Tree]\label{prop:if_invar_holds_for_tree_then_min_p-set_returned_by_tree}
		If $\Inv(\bc,\tuple{\ba,\bb})$ holds for $\scQX'(\bc,\tuple{\ba,\bb})$, then it returns a minimal $p$-set wrt.\ $\tuple{\ba,\bb}$.	
	\end{proposition}
	\begin{proof}
		We prove this proposition by induction on $d$ where $d$ is the maximal number of \emph{recursive}\footnote{That is, \emph{additional} calls made, \emph{not} taking into account the running routine $\scQX'(\bc,\tuple{\ba,\bb})$ that we consider in the proposition.} calls of $\scQX'$ on the call stack throughout the execution of $\scQX'(\bc,\tuple{\ba,\bb})$.\vspace{5pt}
		
		\noindent\emph{Induction Base:} Let $d=0$. That is, no recursive calls are executed, or, equivalently, $\scQX'(\bc,\tuple{\ba,\bb})$ returns in line~\ref{algoline:return_emptyset} or \ref{algoline:return_O}. Since $\Inv(\bc,\tuple{\ba,\bb})$ is true, a minimal $p$-set wrt.\ $\tuple{\ba,\bb}$ is returned, which follows from Prop.~\ref{prop:correctness_of_QX'_if_invar_and_propagation_of_invar}.\ref{prop:correctness_of_QX'_if_invar_and_propagation_of_invar:enum:return_line10_correct}+\ref{prop:correctness_of_QX'_if_invar_and_propagation_of_invar:enum:return_line12_correct}.\vspace{5pt}
		
		\noindent\emph{Induction Assumption:} Let the statement of the proposition be true for $d=k$. We will now show that, in this case, the statement holds for $d = k+1$ as well.\vspace{5pt}
		
		\noindent\emph{Induction Step:} 
		Assume that (at most) $k+1$ recursive calls are ever on the call stack while $\scQX'(\bc,\tuple{\ba,\bb})$ executes. Since $\Inv(\bc,\tuple{\ba,\bb})$ holds, Prop.~\ref{prop:correctness_of_QX'_if_invar_and_propagation_of_invar}.\ref{prop:correctness_of_QX'_if_invar_and_propagation_of_invar:enum:propagation_of_invar_to_recursive_call} lets us conclude that $\Inv(\dmc,\langle\dma,\dmb\rangle)$ holds for the first recursive call $\scQX'(\dmc,\langle\dma,\dmb\rangle)$ issued in line~\ref{algoline:recursive_call1} of $\scQX'(\bc,\tuple{\ba,\bb})$. Now, we have that, for $\scQX'(\dmc,\langle\dma,\dmb\rangle)$, the maximal number of recursive calls ever on the call stack while it executes, is (at most) $k$. Therefore, by the \emph{Induction Assumption}, $\scQX'(\dmc,\langle\dma,\dmb\rangle)$ returns a minimal $p$-set wrt.\ $\langle\dma,\dmb\rangle$. 
		
		Because $\Inv(\bc,\tuple{\ba,\bb})$ holds and $\scQX'(\dmc,\langle\dma,\dmb\rangle)$ called in line~\ref{algoline:recursive_call1} during the execution of $\scQX'(\bc,\tuple{\ba,\bb})$ returns a minimal $p$-set wrt.\ $\langle\dma,\dmb\rangle$, we deduce by means of Prop.~\ref{prop:invar_propagates_to_first_call_in_line17} that $\Inv(\ddmc,\langle\ddma,\ddmb\rangle)$ holds for the call $\scQX'(\ddmc,\langle\ddma,\ddmb\rangle)$ made in line~\ref{algoline:recursive_call2} during the execution of $\scQX'(\bc,\tuple{\ba,\bb})$. Again, it must be true that the maximal number of recursive calls ever on the call stack while $\scQX'(\ddmc,\langle\ddma,\ddmb\rangle)$ executes is (at most) $k$. Consequently, $\scQX'(\ddmc,\langle\ddma,\ddmb\rangle)$ returns a minimal $p$-set wrt.\ $\langle\ddma,\ddmb\rangle$ due to the \emph{Induction Assumption}.
		
		As both recursive calls made throughout the execution of $\scQX'(\bc,\tuple{\ba,\bb})$ return a minimal $p$-set wrt.\ their given $p$-PIs $\langle\dma,\dmb\rangle$ and $\langle\ddma,\ddmb\rangle$, respectively, we conclude by Prop.~\ref{prop:if_both_subtrees_sound_then_tree_is_sound} that $\scQX'(\bc,\tuple{\ba,\bb})$ returns a minimal $p$-set wrt.\ $\tuple{\ba,\bb}$.
		
		This completes the inductive proof.
	\end{proof}
	\begin{theorem}[Correctness of QX]
		Let $\tuple{\ma,\mb}$ be a $p$-PI. Then, $\scQX(\tuple{\ma,\mb})$ returns a minimal $p$-PI wrt.\ $\tuple{\ma,\mb}$ if a $p$-set exists for $\tuple{\ma,\mb}$. Otherwise, $\scQX(\tuple{\ma,\mb})$ returns 'no $p$-set'. 
	\end{theorem}
	\begin{proof}
		%The statement follows immediately from 
		Prop.~\ref{prop:QX_correct_for_trivial_cases}.\ref{prop:QX_correct_for_trivial_cases:enum:no_p-set}, first, proves that 'no $p$-set' is returned if there is no $p$-set wrt.\ $\tuple{\ma,\mb}$. Second, it shows that, if there is a $p$-set wrt.\ $\tuple{\ma,\mb}$, $\scQX$ will either return in line~\ref{algoline:emptyset} or call $\scQX'$ in line~\ref{algoline:call_QX'}. 
		
		We now show that, in both of these cases, $\scQX$ returns a minimal $p$-set wrt.\ $\tuple{\ma,\mb}$. This then implies that a minimal $p$-set is returned by $\scQX$ whenever such a one exists.
		
		First, if $\scQX$ returns in line~\ref{algoline:emptyset}, then the output is a minimal $p$-set wrt.\ $\tuple{\ma,\mb}$ due to Prop.~\ref{prop:QX_correct_for_trivial_cases}.\ref{prop:QX_correct_for_trivial_cases:enum:emptyset}. 
		
		Second, if $\scQX$ calls $\scQX'(\bc,\tuple{\ba,\bb})$ in line~\ref{algoline:call_QX'}, then $\Inv(\bc,\ba,\bb)$ holds according to Prop.~\ref{prop:invar_holds_for_first_call_of_QX'}. Finally, since $\Inv(\bc,\ba,\bb)$ holds for $\scQX'(\bc,\tuple{\ba,\bb})$, Prop.~\ref{prop:if_invar_holds_for_tree_then_min_p-set_returned_by_tree} establishes that $\scQX'(\bc,\tuple{\ba,\bb})$ returns a minimal $p$-set wrt.\ $\tuple{\ma,\mb}$.   
	\end{proof}

	\section{Conclusion}
	\label{sec:conclusion}
	\textsc{QuickXPlain} (QX) is a very popular, highly cited, and frequently employed, adapted, and extended algorithm to solve the MSMP problem, i.e., to find a subset of a given universe such that this subset is irreducible subject to a monotone predicate (e.g., logical consistency). MSMP is an important and common problem and its manifestations occur in a wide range of computer science disciplines. Since QX has in practice turned out to be hardly understood by many---experienced academics included---and was published without a proof, 
	%and since every proposal of an algorithm should be accompanied by a formal proof, 
	we account for that by providing for QX an intelligible
	%, easily-comprehensible 
	\emph{proof that explains}. The availability and accessibility of a formal proof is instrumental in various regards. Beside allowing the verification of QX's correctness (\emph{proof effect}), it fosters proper and full understanding of QX and of other works relying on QX (\emph{didactic effect}), it is a necessary foundation for ``gapless'' correctness proofs of numerous algorithms, e.g., in model-based diagnosis, that rely on (results computed by) QX (\emph{completeness effect}), it makes the intuition of QX's correctness bullet-proof and excludes the later detection of algorithmic errors, as was already experienced even for seminal works in the past (\emph{trust and sustainability effect}), as well as it might be used as a template for devising proofs of other recursive algorithms (\emph{transfer effect}). 
	Since \emph{(i)}~we exemplify the workings of QX using a novel 
	%and (from our experience) 
	tried and tested
	well-comprehensible notation, and 
	\emph{(ii)}~we put a special emphasis on the clarity and didactic value of the given proof (e.g., by segmenting the proof into small, intuitive, and easily-digestible chunks, and by showing how our proof can be ``directly traced'' using the recursive call tree produced by QX), we believe that this work can decisively contribute to a better understanding of QX, which we expect to be of great value for both practitioners and researchers.
	%
	%Because \emph{(i)}~we exemplify the workings of QX using a novel and (from our experience) well-comprehensible notation, \emph{(ii)}~our proof can be ``directly traced'' using the recursive call tree produced by QX, and \emph{(iii)}~we segment the proof into small, easily-digestible chunks in an intuitive way, we believe that this work can decisively contribute to a better understanding of QX, which can be of great value for practitioners and researchers alike.     

	%of a proof comes with pivotal  

	%, and since a proof, beside allowing the verification   
	\section*{Acknowledgments}
	This work was partly supported by the Austrian Science Fund (FWF), contract P-32445-N38.  
	
%	\section*{References}
	
	%\bibliography{library}
	\fontsize { 8pt }{ 9pt } 
	\selectfont

\end{document}